\documentclass{elsarticle}

\usepackage{lineno,hyperref}
\usepackage{amssymb}
\usepackage{cite}
\usepackage{microtype}
\usepackage{amsthm,amssymb}
\usepackage[linesnumbered,ruled,vlined]{algorithm2e}
\usepackage{amsfonts}
\usepackage{dsfont}
\usepackage{float}
\usepackage{subfig}
\usepackage{booktabs}
\usepackage{graphics}
\usepackage{mathtools}
\usepackage{wrapfig}

\usepackage{xspace}

\newtheorem{theorem}{\textbf{Theorem}}

\newtheorem{lemma}{\textbf{Lemma}}

\usepackage{xcolor}

\let\originalleft\left
\let\originalright\right
\renewcommand{\left}{\mathopen{}\mathclose\bgroup\originalleft}
\renewcommand{\right}{\aftergroup\egroup\originalright}

\usepackage{threeparttable}

\newcommand{\mmas}{$\lambda$-\text{MMAS}\xspace}
\newcommand{\pbil}{\text{\sc PBIL}\xspace} 
\newcommand{\eda}{\text{\sc EDA}\xspace} 
\newcommand{\edas}{\text{\sc EDAs}\xspace} 
\newcommand{\eas}{\text{\sc EAs}\xspace} 
\newcommand{\cga}{\text{\sc cGA}\xspace} 
\newcommand{\umda}{\text{\sc UMDA}\xspace} 
\newcommand{\om}{\text{\sc OneMax}\xspace} 
\newcommand{\bval}{\text{\sc BinVal}\xspace} 
\newcommand{\los}{\text{\sc LeadingOnes}\xspace} 
\newcommand{\oneoneEA}{$(1+1)~\text{EA}$\xspace}

\newcommand{\bin}[1]{\text{Bin}\left(#1\right)} 

\usepackage{flushend}

\newcommand{\bigO}[1]{\mathcal{O}\left(#1\right)} 

\usepackage{float}

\journal{arXiv.org}

\begin{document}

\begin{frontmatter}

\title{Runtime Analysis of the Univariate Marginal Distribution
Algorithm under Low Selective Pressure and Prior Noise\tnoteref{mytitlenote}}
\tnotetext[mytitlenote]{Preliminary version of this work will appear
in the Proceedings of the 2019 Genetic and Evolutionary 
Computation Conference (GECCO 2019), Prague, Czech Republic.}

\author{Per Kristian Lehre \& Phan Trung Hai Nguyen}
\address{School of Computer Science, University of Birmingham, Birmingham B15 2TT, U.K.}

\begin{abstract}
We perform a rigorous runtime analysis for the
Univariate Marginal Distribution Algorithm 
on the \los function, a well-known benchmark
function in the theory community of evolutionary computation
with a high correlation between decision variables.
For a problem instance of size $n$, the currently best known 
upper bound on the expected runtime
is $\bigO{n\lambda\log\lambda+n^2}$ 
(Dang and Lehre, GECCO 2015), while a lower bound necessary to
understand how the algorithm copes with variable dependencies is still missing. 
Motivated by this, we 
show that the algorithm requires a $e^{\Omega(\mu)}$ 
runtime with high probability and in expectation
if the selective pressure  is low; otherwise, we obtain a lower bound of $\Omega(\frac{n\lambda}{\log(\lambda-\mu)})$ 
on the expected runtime. Furthermore, 
we for the first time consider the algorithm on the
function under a prior noise model and 
obtain an $\bigO{n^2}$  
expected runtime for the optimal parameter settings.
In the end, our theoretical 
results are accompanied by empirical findings, not only matching
with rigorous analyses but also providing
new insights into the behaviour of the algorithm. 
\end{abstract}


\end{frontmatter}

\section{Introduction}

Estimation of Distribution Algorithms (\edas)
\citep{Muhlenbein1996,Pelikan2002,larranaga2001estimation}
are black-box optimisation methods
that search for optimal solutions by building and sampling
from probabilistic models. They are known by various other names, including 
probabilistic model-building genetic algorithm or
iterated density estimation algorithms.
Unlike traditional evolutionary algorithms (\eas), which use standard
genetic operators such as mutation and crossover to create variations, \edas, on the other hand,
achieve it via model building and model sampling. 
The workflow of \edas is an iterative process.
The starting model is a uniform distribution over the search space, from which the initial population of $\lambda$ individuals is sampled. 
The fitness function then scores each individual, 
and the algorithm selects the $\mu<\lambda$ fittest individuals to update the model. 
The procedure is
repeated many times and terminates
when a threshold on the number of iterations is exceeded or
a solution of good quality is obtained \citep{Eiben:2003:IEC:954563, bib:Hauschild2011}.
We call the value $\lambda$ the offspring population size, while 
the value $\mu$ is known as the parent population size of the algorithms.

Several \edas have been proposed over the last
decades. They differ in how they learn the variable interactions and build/update the probabilistic models over iterations.
In general, \edas can be categorised
into two  classes: univariate and multivariate. 
Univariate \edas, which take advantage of first-order statistics 
(i.e. the means while assuming variable independence),
usually represent the 
model as a probability vector, and 
individuals are sampled independently and identically from a 
product distribution. 
Typical \edas in this class are the 
Univariate Marginal Distribution Algorithm (\umda \citep{Muhlenbein1996}), 
the compact Genetic Algorithm (\cga \citep{bib:Harik1997}) 
and the Population-Based Incremental Learning (\pbil \citep{bib:Baluja1994}). 
Some ant colony optimisation algorithms like the \mmas \citep{STUTZLE2000889}
can also be cast into this framework 
(also called $n$-\text{Bernoulli}-$\lambda$-$\eda$ 
 \citep{bib:Friedrich2016}). 
In contrast, multivariate \edas apply statistics of order two or more to capture the underlying structures of the addressed problems. 
This paper focuses on univariate \edas on discrete 
optimisation, and for that reason
we refer the interested readers to 
\citep{bib:Hauschild2011, larranaga2012review} 
for other \edas on a continuous domain.

In the theory community, researchers perform rigorous analyses
to gain insights into  the  runtime
(synonymously, optimisation time), which is defined as 
the number of function evaluations of the
algorithm until an optimal solution is found for the 
first time. In other words, theoretical work 
usually addresses the unlimited case
when we consider the run of the algorithm as an infinite process.
Considering function evaluations is motivated by the fact that these are
often the most expensive operations, whereas other operations can usually be
executed very quickly. Steady-state algorithms like the simple 
\oneoneEA have the number of function evaluations equal the number of iterations,
whereas for univariate \edas the former is larger by a factor of the offspring population size $\lambda$ than the latter. 
Runtime analyses give performance guarantee of the algorithms for a wide range of problem instance sizes. 
Due to the complex interplay of variables and limitations on 
the state-of-the-art tools in algorithmics, runtime analysis is often performed on 
simple (artificial) problems such as \om, \los and \bval, hoping that this 
provides valuable insights into the development of 
new techniques for analysing search heuristics and the behaviour of such algorithms 
on easy parts of more complex
problem spaces \citep{Doerr2016}. By 2015, there had been a handful of runtime results 
for \edas \citep{Dang:2015,bib:Friedrich2016}, since then this class of 
algorithms have constantly drawn more attention from the 
theory community as evidenced in the
increasing number of \eda-related publications recently \citep{bib:Friedrich2016,bib:Krejca, bib:Sudholt2016,bib:Wu2017,bib:Witt2017,CGA-GAUSSIAN-NOISE,bib:Lehre2017,bib:lehre2018a,
Lehre2018PBIL,Witt2018Domino,
bib:Lengler2018,DoerrSigEDA2018,Hasenohrl:2018}.

Droste \citep{bib:Droste2006} in 2006
performed the first rigorous 
analysis of the \cga,
which works on a population of two individuals and updates the
probabilistic model additively via 
a parameter $K$ (also referred to as the hypothetical
population size of a genetic algorithm 
that the \cga is supposed to model) and obtained 
a lower bound  $\Omega(K\sqrt{n})=\Omega(n^{1+\varepsilon})$ 
on the expected runtime for
the \cga on any pseudo-Boolean function 
for any small constant $\varepsilon>0$. 
Each component in the probabilistic model (also called marginal) 
of the \cga considered in \citep{bib:Droste2006} 
is allowed to reach the extreme values zero and one. Such an
algorithm is referred to as 
an \eda without margins, since in contrast it is 
possible to reinforce the margins $[1/n,1-1/n]$
(sometimes called borders) to keep it 
away from the extreme probabilities.
Friedrich \textit{et al.} \citep{bib:Friedrich2016}, on consideration of
univariate \edas (without borders),
conjectured that the \cga might not 
optimise the \los function efficiently (i.e., within an $\bigO{n^2}$
expected runtime) as the algorithm is 
balanced but not stable. They then proposed 
a so-called stable \cga to overcome this, which requires an 
$\bigO{n\log n}$ expected runtime on the same function. 
Motivated by the same work, 
Doerr \textit{et al.} \citep{DoerrSigEDA2018} 
recently developed the significant \cga, 
which uses memory to determine 
when the marginals should be set to a value in the set 
$\{1/n,1/2,1-1/n\}$, and   
surprisingly the algorithm 
optimises the \om and \los functions using an $\bigO{n\log n}$ 
expected runtime.

The \umda is probably the most famous univariate 
\eda. In each so-called iteration, the algorithm updates each marginal
to the corresponding frequency of 1s among the $\mu$ 
fittest individuals. In 2015,
Dang and Lehre \citep{Dang:2015} via the level-based theorem
\citep{bib:Corus2016} obtained an upper bound of
$\bigO{n\lambda\log \lambda + n^2}$ on the expected runtime 
for the algorithm (with margins) on the \los function when the offspring
population size is $\lambda=\Omega(\log n)$ and 
the selective pressure $\mu/\lambda\le  1/(1+\delta)e$
for any constant $\delta>0$.
For the optimal setting $\lambda=\bigO{n/\log n}$, 
the above bound becomes
$\bigO{n^2}$, which emphasises the need of borders for the algorithm 
to optimise the \los function efficiently 
compared to the findings in \citep{bib:Friedrich2016}.
We also note that a generalisation of the \umda 
is the \pbil  \citep{bib:Baluja1994}, which  
updates the marginals using a convex combination with a 
smoothing parameter $\eta\in [0,1]$ between
the current marginals and the frequencies of 1s among the 
$\mu$ fittest individuals
in the current population. 
Wu \textit{et al.} \citep{bib:Wu2017} performed the first 
rigorous runtime analysis of the algorithm, where
they argued that for a sufficiently large 
population size, the algorithm can
avoid making wrong decisions early even when 
the smoothing parameter is large. They 
also showed an upper bound $\bigO{n^{2+\epsilon}}$ on the 
expected runtime for the \pbil (with margins) on the \los
function for some small constant $\varepsilon>0$. 
The required offspring population size yet 
still remains large \citep{bib:Wu2017}. Very recently,
Lehre and Nguyen \citep{Lehre2018PBIL}, via the level-based theorem
 with some additional arguments, obtained an upper bound
 $\bigO{n\lambda\log \lambda+n^2}$ on the expected runtime 
for the offspring population 
sizes $\lambda=\Omega(\log n)$ and a sufficiently high selective pressure. 
This result improves the bound in \citep{bib:Wu2017} by a factor of $\Theta(n^\varepsilon)$ for the optimal parameter 
setting $\lambda=\bigO{n/\log n}$.

In this paper, we analyse the \umda in order to, when combining with previous results \citep{Dang:2015,bib:lehre2018a}, 
completes the picture on the runtime of the algorithm on the \los function, a widely used
benchmark function with a high correlation between variables. 
We first show that under a low selective pressure the algorithm fails to optimise the function in polynomial runtime with high probability and in expectation. 
This result essentially reveals the limitations of probabilistic models
based on probability vectors as the algorithm hardly stays
in promising states 
when the selective pressure is not high enough, while the global optimum
cannot be sampled with high probability.
On the other hand, when 
the selective pressure is sufficiently high, we obtain a lower bound of 
$\Omega(\frac{n\lambda}{\log(\lambda-\mu)})$ 
on the expected runtime for the offspring population sizes 
$\lambda=\Omega(\log n)$.
Moreover, we introduce noise to the \los function, 
where a uniformly chosen bit is flipped with (constant)
probability $p<1$ before evaluating the fitness 
(also called prior noise). Via the level-based theorem, 
we show that the expected runtime of the algorithm on the noisy function 
is still $\bigO{n^2}$ for an optimal population size $\lambda=\bigO{n/\log n}$. 
To the best of our knowledge, this is the first 
time that the \umda is rigorously studied in 
a noisy environment, while the \cga 
is already considered in \citep{CGA-GAUSSIAN-NOISE} under Gaussian posterior noise.
Despite the simplicity of the noise model, 
this can be viewed as the first step towards
understanding the behaviour of the algorithm in a noisy environment. 
In the end, we provide empirical 
results to support our theoretical analyses and give new
insights into the run of the algorithm which the theoretical results do not cover.
Moreover, many algorithms similar to the \umda with 
a fitness proportional selection are popular in
bioinformatics \citep{Armaaanzas2008}, where they 
relate to the notion of \textit{linkage equilibrium}
\citep{bib:Slatkin2008,bib:muhlenbein2002} 
-- a popular model assumption in 
population genetics. Therefore, studying the \umda especially 
in the presence of variable dependence and mild noise
solidifies our
understanding of population dynamics.

The paper is structured as follows.
Section~\ref{sec:preliminaries} introduces
the studied algorithm. Section~\ref{sec:umda-on-los-low} provides a 
detailed analysis for the algorithm on
the \los function in case of low selective pressure, 
followed by the analysis for a  high 
selective pressure in Section~\ref{sec:umda-los-high-sel-press}.
In Section~\ref{sec:umda-noise-los}, we 
introduce the \los function
with prior noise and show an upper bound $\bigO{n^2}$ on the expected 
runtime. Section~\ref{sec:experiments} presents an empirical study to 
complement theoretical results derived earlier. 
The paper ends in Section~\ref{sec:conclusion}, 
where we give our concluding remarks and speak of potential future work. 

\section{The algorithm}
\label{sec:preliminaries}

In this section we describe the studied algorithm. Let
$\mathcal{X}=\{0,1\}^n$ be 
a finite binary search space with $n$ dimensions, 
and each individual in $\mathcal{X}$ is represented as 
$x=(x_1,x_2,\ldots,x_n)$. 
The population of $\lambda$ individuals 
in iteration $t$ is denoted as 
$P_t:=(x_t^{(1)},\ldots,x_t^{(\lambda)})$. 
We consider the maximisation of  
an objective function
$f:\mathcal{X} \rightarrow \mathbb{R}$.

The \umda, defined in Algorithm~\ref{umda-algor},
maintains a probabilistic model that is 
represented as an $n$-vector 
$p_t:=(p_{t,1},\ldots,p_{t,n})$, and 
each marginal $p_{t,i}\in [0,1]$ 
for $i \in [n]$ (where $[n]:=[1,n]\cap \mathbb{N}$)  is
the probability of sampling a one at the $i$-th 
bit position in the offspring. The 
joint probability distribution of an individual $x\in \mathcal{X}$
given the current model $p_t$ is formally defined as
\begin{equation}\label{eq:product-distribution}
\Pr\left(x \mid p_t\right)
=\prod_{i=1}^{n}\left(p_{t,i}\right)^{x_i} \left(1-p_{t,i}\right)^{1-x_i}.
\end{equation}
The starting model is the uniform 
distribution $p_0:=(1/2,\ldots,1/2)$. In an iteration $t$, 
the algorithm samples
a population $P_t$ of $\lambda$ individuals, sorts them
in descending order according to fitness and then selects the $\mu$ fittest individuals to
update the model (also called the selected population).
Let $X_{t,i}$ denote the number of 1s sampled
at bit position $i\in [n]$ in the selected population.
The algorithm updates each marginal using  
$p_{t+1,i} = X_{t,i}/\mu$. 
Each marginal is also restricted to be within the interval $[1/n,1-1/n]$,
where the values $1/n$ and $1-1/n$ are called 
lower and upper border, respectively.
We call the ratio $\gamma^*:=\mu/\lambda$  
the selective pressure of the algorithm.

\begin{algorithm}
    \DontPrintSemicolon        
        $t\leftarrow 0$; initialise $p_t\leftarrow (1/2,1/2,\ldots,1/2)$\;
        \Repeat{termination condition is fulfilled}{
            \For{$j=1,2,\ldots,\lambda$}{
                sample $x_{t,i}^{(j)} \sim \text{Bernoulli}(p_{t,i})$ 
                for each $i\in [n]$ \;
            }
            sort $P_t \leftarrow (x_t^{(1)},x_t^{(2)},\ldots,x_t^{(\lambda)})$ such that
            $f(x^{(1)})\ge f(x^{(2)})\ge \ldots\ge f(x^{(\lambda)})$\;
            \For{$i=1,2,\ldots,n$}{
                $p_{t+1,i} \leftarrow \max\{1/n, \min\{1-1/n, X_{t,i}/\mu\}\}$\;
              }
            $t\leftarrow t+1$\;
        }
    \caption{\umda \label{umda-algor}}
\end{algorithm}

\section{Low Selective Pressure}
\label{sec:umda-on-los-low}

Recall that we consider the problem of maximising 
the number of leading 1s in a 
bitstring, which is  defined by 
$$
\los(x):=\sum_{i=1}^{n}\prod_{j=1}^{i}x_j.
$$ 
The bits in this particular function are highly correlated, 
so it is often used to study the ability of \eas 
to cope with variable dependency \citep{Krejca2018survey}. 
Previous studies \citep{Dang:2015,bib:lehre2018a} showed that the \umda 
 optimises the function within an $\bigO{n^2}$ expected time 
 for the optimal offspring  population size 
$\lambda=\bigO{n/\log n}$.

Before we get to analysing the function, we introduce 
some notation.
Let $C_{t,i}$ for all $i\in [n]$ 
denote the number of individuals
having at least $i$ leading 1s
in iteration $t$, and
$D_{t,i}$ is the number of individuals having 
$i-1$ leading 1s, 
followed by a 0 at the block $i$. For the special case of $i=1$,
$D_{t,i}$ consists of those with zero leading 1s.
Furthermore, let $(\mathcal{F}_t)_{t\in \mathbb{N}}$
be a filtration induced from the 
population $(P_t)_{t\in \mathbb{N}}$.

Once the population has been sampled, 
the algorithm invokes truncation selection to select the $\mu$ fittest individuals to update the probability vector. 
We take this $\mu$-cutoff into account
by defining a random variable 
$$
Z_t := \max\{i\in \mathbb{N} : C_{t,i}\ge \mu\},
$$ 
which tells us 
how many marginals, counting from position one, 
are set to the upper border $1-1/n$
in iteration $t$. Furthermore, we define another random variable
$$
Z_t^*:=\max\{i\in \mathbb{N} : C_{t,i}>0\}
$$
to be the number of leading 1s
of the fittest individual(s).
For readability, we often leave out the indices of random variables like when we write $C_{t}$ instead of $C_{t,i}$, 
if values of the indices are clear from the context. 

\subsection{On the distributions of $C_{t,i}$ and $D_{t,i}$}

In order to analyse the distributions of the random variables
$C_{t,i}$ and $ D_{t,i}$, 
we shall take an alternative view on the 
sampling process at an arbitrary 
bit position $i\in [n]$ in iteration $t\in \mathbb{N}$
via the \textit{principle of deferred decisions}
\citep{bib:Motwani1995}.
We imagine that the process 
samples the values of the first bit for $\lambda$
individuals. Once this has finished, it moves on to
the second bit and so on until the population is sampled. 
In the end, we will obtain
a population that is sorted in descending order according to fitness.

We now look at the first bit 
in iteration $t$. The number of 1s sampled 
in the first bit position 
follows a binomial distribution with 
parameters $\lambda$ and $p_{t,1}$, i.e.,
$C_{t,1}\sim \bin{\lambda,p_{t,1}}$.
Thus, the number of 0s at the first bit position is 
$D_{t,1}=\lambda-C_{t,1}$.

Having sampled the first bit 
for $\lambda$ individuals, and 
note that the bias due to selection
in the second bit position comes into play only if 
the first bit is 1. 
If this is the case, then a 1 is more preferred to a 0. 
The probability of sampling a 1 
is $p_{t,2}$; thus, the number of individuals having at least 
2 leading 1s is binomially distributed with 
parameters $C_{t,1}$ and $p_{t,2}$, that is, 
$C_{t,2}\sim \bin{C_{t,1},p_{t,2}}$, 
and the number of 
0s equals $D_{t,2}=C_{t,1}-C_{t,2}$. Unlike the first
bit position, there are still $D_{t,1}$
remaining individuals, since for these individuals 
the first bit is a 0, there is no bias
between a 1 and a 0. The number of 1s 
follows a binomial distribution with 
parameters $D_{t,1}$ (or $\lambda-C_{t,1}$) 
and $p_{t,2}$.

We can generalise this result for an arbitrary bit position 
$i\in [n]$. The number of individuals having at least 
$i$ leading 1s follows a binomial distribution with 
$C_{t,i-1}$ trials and success probability $p_{t,i}$,
i.e.,
$C_{t,i} \sim \bin{C_{t,i-1},p_{t,i}},$
and $D_{t,i}= C_{t,i-1}-C_{t,i}.$
Furthermore, the number of 1s sampled among the 
$\lambda-C_{t,i-1}$ remaining individuals is 
binomially distributed with $\lambda-C_{t,i-1}$ 
trials and success probability $p_{t,i}$.
If we consider the expectations of these random variables,
 by the tower rule 
 \citep{bib:Feller1968} and noting that 
 $p_{t,i}$ is $\mathcal{F}_{t-1}$-measurable, 
 we then get
  \begin{align}
  \begin{split}\label{eq:los-expectation-of-C}
  \mathbb{E}[C_{t,i}\mid \mathcal{F}_{t-1}]
  &=\mathbb{E}[\mathbb{E}[C_{t,i}\mid C_{t,i-1}]\mid \mathcal{F}_{t-1}]
  =  \mathbb{E}[C_{t,i-1}\mid \mathcal{F}_{t-1}]\cdot p_{t,i}, 
  \end{split}
  \end{align}
  and similarly
  \begin{align}
      \begin{split}\label{eq:los-expectation-of-D}
  \mathbb{E}[D_{t,i}\mid \mathcal{F}_{t-1}] &
  =\mathbb{E}[C_{t,i-1}\mid \mathcal{F}_{t-1}]\cdot (1-p_{t,i}).
  \end{split}  
  \end{align}
  
We aim at showing that the \umda takes exponential time
to optimise the \los function when the selective pressure is not sufficiently high,
as required in \citep{bib:lehre2018a}. Later analyses are concerned with
two intermediate values:
  \begin{align}
  \begin{split}\label{eq:alpha-value}
  \alpha=\alpha(n)&:=\log(\gamma^*/(1-\delta))/\log (1-1/n) \\    
  \end{split}\\
  \begin{split}\label{eq:beta-value}
  \beta=\beta(n)&:=\log(\gamma^*/(1+\delta))/\log (1-1/n)\\
  \end{split}  
  \end{align}
for any constant $\delta\in (0,1)$. Clearly, we always get $\alpha\le \beta$.
We also define a stopping time 
$\tau:=\min\{t\in \mathbb{N} \mid Z_t\ge \alpha\}$ to be the 
first hitting time of the value $\alpha$ for the random variable $Z_t$.
We then consider two phases: (1) until the random variable 
$Z_t$ hits the value $\alpha$ 
for the first time ($t\le \tau$), 
and (2) after the random variable 
$Z_t$ has hit the value $\alpha$ for the first time ($t>\tau$). 

\subsection{Before $Z_t$ hits value $\alpha$ for the first time}

The algorithm starts with an 
initial population $P_0$ sampled from a 
uniform distribution $p_0=(1/2,\ldots,1/2)$. 
An initial observation is that
the all-ones bitstring cannot be sampled in the population $P_0$ 
with high probability since 
the probability of sampling 
it from the uniform distribution 
is $2^{-n}$, then by the union bound \citep{bib:Motwani1995} 
it appears in the population $P_0$ 
with probability at most 
$\lambda\cdot 2^{-n} =2^{-\Omega(n)}$ since we 
only consider the offspring population of 
size at most polynomial in the problem instance size $n$.
The following lemma states the expectations of the random variables
$Z_0^*$ and $Z_0$.

\begin{lemma}
\label{lemma:fittest-individual}
    $\mathbb{E}[Z_0^*]=\bigO{\log \lambda}$, and
    $\mathbb{E}[Z_0] = \bigO{\log (\lambda-\mu)}$.
\end{lemma}

The proof uses that the random variables $Z_0^*$ and $Z_0$ denote 
the expected numbers of leading 1s of the fittest 
individual in populations of $\lambda$ and $\lambda-\mu$ individuals, 
respectively, sampled from a uniform distribution and by a 
result in \citep{EISENBERG2008135}.
We now show that the value of the random variable $Z_t$ 
never decreases during phase 1 with high 
probability by noting that its value gets decreased if the number of 
individuals with at least $Z_t$ leading 1s 
in iteration $t+1$ is less than $\mu$. 

\begin{lemma}
\label{lemma:prod-sampling-more-alpha}
$\Pr\left(\forall t\in  [1,\tau] : Z_t\ge Z_{t-1}\right)\ge 1-\tau e^{-\Omega(\mu)}$.
\end{lemma}
\begin{proof}
    We will show via strong induction on time step $t$ 
    that the probability that there 
    exists an iteration $t\in [1,\tau]$ such that
    $Z_t< Z_{t-1}$ is at most $\tau e^{-\Omega(\mu)}$.
    The base case $t=1$ is trivial since $Z_{t-1}=Z_0$ and 
     the probability of sampling at most $\mu$ individuals having 
    at least $Z_0$ leading 1s 
    is at most $e^{-\Omega(\mu)}$. This is because 
     $Z_t< \alpha$ for all $t< \tau$, in expectation 
there are at least 
$(1-1/n)^{Z_t}\lambda \ge 
(1-1/n)^\alpha \lambda 
= \mu/(1-\delta)$ individuals with at least $Z_t$ leading 
1s sampled in iteration $t+1$. By a Chernoff bound \citep{bib:Motwani1995}, 
the probability of
sampling at most $(1-\delta)\cdot \mu/(1-\delta)=\mu$ such individuals
is at most $e^{-(\delta^2/2)\cdot \mu/(1-\delta)}=e^{-\Omega(\mu)}$ for any constant
$\delta\in (0,1)$.  

For the inductive step, we assume that the result
holds for the first $t<\tau$ iterations, meaning that 
$\Pr(\exists t'\le t:Z_{t'}< Z_{t'-1})\le  t e^{-\Omega(\mu)}$. 
We are left to show that it also holds 
for iteration $t+1$, that is,
$\Pr(\exists t'\le t+1:Z_{t'}< Z_{t'-1})\le  (t+1) e^{-\Omega(\mu)}$. 
Again, by Chernoff bound, there are at most 
$\mu$ individuals with at least $Z_t<\alpha$  
leading 1s in iteration $t+1$ with probability
at most $e^{-\Omega(\mu)}$. 
By a union bound, this rare event does not happen during the first 
$t+1$ iterations with probability at most $(t+1)e^{-\Omega(\mu)}$, which
completes the inductive step, and the lemma follows.
\end{proof}

\subsection{After $Z_t$ has hit value $\alpha$ for the first time}

The preceding section shows that the random variable
$Z_t$ is non-decreasing during phase 1
with probability $1-\tau e^{-\Omega(\mu)}$.
The following lemma also shows that its value
stays above $\alpha$ afterwards
with high probability.
  
  \begin{lemma}
  \label{lemma:stay-above-alpha}
  For any constant $k>0$, it holds that
   $$\Pr(\forall t\in[\tau,\tau+e^{k\mu}]:Z_t\ge \alpha(n))
   \ge 1- e^{k\mu}\cdot e^{-\Omega(\mu)}.$$
  \end{lemma}

Recall that we aim at showing an $e^{\Omega(\mu)}$ lower
bound on the runtime, so we assume that the stopping time 
$\tau$ is at most $e^{\Omega(\mu)}$. Otherwise, if this assumption does
not hold and the selective pressure is sufficiently 
low (as chosen below) such that $n-\alpha\ge n-\beta =\Omega(n)$,
then we are done. 
The following lemma further shows that
there is also an upper bound on the random variable $Z_t$. 

\begin{lemma}
\label{lemma:stay-below-beta}
For any constant $k>0$, it holds that
$$
\Pr(\forall t \in [1,e^{k\mu}]: Z_{t} \le \beta(n))
\ge 1-e^{k\mu}\cdot e^{-\Omega(\mu)}.
$$
\end{lemma}
\begin{proof}
    It suffices to show that
    $\Pr(\exists t \in [1,e^{k\mu}]:Z_t > \beta)
    \le e^{k\mu}\cdot e^{-\Omega(\mu)}$
via strong induction on time step $t$.
The base case $t=1$ is trivial. For the inductive step, we assume that 
the result holds for the first $t$ iterations and need
to show that it also holds 
for iteration $t+1$, meaning that 
$\Pr\left(\exists t'\le t+1:Z_{t'}> \beta\right)
\le (t+1)e^{-\Omega(\mu)}$.  With probability 
$t e^{-\Omega(\mu)}$, we get $Z_{t'}> \beta$ in an iteration $t'\le t$, 
and if this rare event does not happen, then we obtain $Z_t\le \beta$, meaning that
in the best case the first $\beta$ marginals are set to the upper border $1-1/n$.
Then, in iteration $t+1$, the expected number of individuals with at least 
$\beta$ leading 1s is  
$
\lambda (1-1/n)^{\beta}=\mu/(1+\delta)
$ for some constant $\delta\in (0,1)$.
By Chernoff bound, the probability of sampling at least 
$(1+\delta)\cdot \mu/(1+\delta)=\mu$
such individuals is at most 
$e^{-(\delta^2/3)\cdot \mu/(1+\delta)}=e^{-\Omega(\mu)}$.
By the union bound, the probability that $Z_{t'}>\beta$ 
for an iteration $t'\le t+1$ is at most $(t+1)e^{-\Omega(\mu)}$. Thus, the inductive
step is complete, and the lemma itself passes.
\end{proof}

Lemma~\ref{lemma:stay-above-alpha} and Lemma~\ref{lemma:stay-below-beta} 
together give essential insights about the behaviour of the algorithm.
The random variable $Z_t$ will stay well below
the threshold $\beta(n)$ for $e^{\Omega(\mu)}$  iterations with probability 
$1-e^{-\Omega(\mu)}$ for 
 a sufficiently large parent population size $\mu$. 
More precisely, the random variable $Z_t$ will fluctuate
around an \textit{equilibrium} value 
$\kappa=\kappa(n):=\log(\gamma^*)/\log(1-1/n)$. This is because when
$Z_t=\kappa$, in expectation there are 
exactly $\lambda(1-1/n)^{\kappa} =\lambda\gamma^*=\mu$ individuals
having at least $\kappa$ leading 1s. 

Furthermore, an exponential lower bound on the runtime is obtained if we can also show that the probability of sampling the $n-\beta$ remaining bits correctly is exponentially small. 
We now choose the selective pressure $\gamma^*$ such that  
$n-\beta\ge \varepsilon n$ for any constant $\varepsilon\in (0,1)$, 
that is equivalent to $\beta\le n(1-\varepsilon)$.
By (\ref{eq:beta-value}) and solving for 
$\gamma^*$, we then obtain 
$
\gamma^*/(1+\delta) 
\ge \left[\left(1-1/n\right)^{n}\right]^{1-\varepsilon}.
$
The right-hand side is at most $1/e^{1-\varepsilon}$ 
as $(1-1/n)^n\le 1/e$ for all $n>0$ \citep{bib:Motwani1995}, 
so the above inequality
always holds if the selective pressure 
satisfies $\gamma^*\ge (1+\delta)/e^{1-\varepsilon}$ 
for any constants $\delta>0$ and $\varepsilon\in (0,1)$.

The remainder of this section shows that the $n-(\beta+1)=\Omega(n)$ 
remaining bits cannot be sampled correctly in any polynomial number of iterations with high probability. 
We define  $(Y_{t,i})_{i\in [n]}$ to be  an offspring 
sampled from the probabilistic model $p_t$. The following lemma shows 
that the sampling process among the $\Omega(n)$ remaining bits are indeed
independent.

\begin{lemma}
\label{lemma:los-Y-independent}
For any $t\le e^{k\mu}$ for any constant $k>0$, 
with probability $1-e^{k\mu}\cdot e^{-\Omega(\mu)}$ 
that the random variables
$(Y_{t,i})_{i\ge \beta+2}$ are pairwise independent.
\end{lemma}

The proof uses that the number 
of 1s sampled in the selected population
at any bit position between $Z_t+2\le \beta+2$ 
(with high probability by Lemma~\ref{lemma:stay-below-beta}) 
and $n$
is binomially distributed
with $\mu$ trials and its marginal probability. 
This is because conditional on the random variable $C_{t,i}$ for $i=Z_t+1$ 
the number of 1s in the selected population in bit 
position $i+1$ can be written as
$\bin{C_{t,i},p_{t,i+1}}+\bin{\mu-C_{t,i},p_{t,i+1}}=\bin{\mu,p_{t,i+1}}$,
independent of the random variable $C_{t,i}$. The second
term in the above sum results from the fact that there is no bias at bit position 
$i+1$ among the $\mu-C_{t,i}$ remaining individuals in the selected population.

For any $i\ge \beta+1$, we always get
$\mathbb{E}[Y_{t,i} \mid \mathcal{F}_{t-1}]
=p_{t,i}$, and again by the tower rule
$\mathbb{E}[Y_{t,i}]
=\mathbb{E}[\mathbb{E}[Y_{t,i}\mid \mathcal{F}_{t-1}]]
=\mathbb{E}[p_{t,i}]$. 
For the \umda without margins, we obtain 
$\mathbb{E}[p_{t,i}]=1/2$ since 
$(p_{t,i})_{t\in \mathbb{N}}$ is a 
martingale \citep{bib:Friedrich2016}
and the initial value $p_{0,i}=1/2$, resulting that
$\mathbb{E}[Y_{t,i}]=1/2$ for all $t\in \mathbb{N}$. 
However, when borders are taken into account, 
 $\mathbb{E}[Y_{t,i}]$ no longer exactly equals but 
remains very close to the value $1/2$. The following lemma 
shows that the expectation of an arbitrary marginal
$i\ge \beta+2$ stays within $(1\pm o(1))(1/2)$ 
for any $t\le e^{\Omega(\mu)}$ with high probability.

\begin{lemma}
\label{lemma:los-expectation-of-Y}
Let $\mu\ge c\log n$ for a sufficiently large
constant $c>0$. Then, it holds 
with probability $1-e^{-\Omega(\mu)}$ that
$\mathbb{E}[p_{t,i}] = (1\pm o(1))(1/2)$
for any $t\le e^{\Omega(\mu)}$ and any $i\ge \beta+2$.
\end{lemma}

Lemma~\ref{lemma:los-expectation-of-Y}
gives us insights
into the expectation of the 
marginal at any time $t\le e^{\Omega(\mu)}$.
One should not confuse the expectation
with the actual value of the marginals. Friedrich \textit{et al.} 
\citep{bib:Friedrich2016} showed that  
even when the expectation stays at
$1/2$ (for the \umda without borders), 
the actual value of the marginal in iteration $t$
can fluctuate close to the trivial lower or upper border due to 
its large variance. 
\begin{lemma}
\label{lemma:los-unlikely-sample-remaining-blocks}
Let $\mu\ge c\log n$ for some sufficiently large constant 
$c>0$ and $\gamma^*\ge (1+\delta)/e^{1-\varepsilon}$ for any constants
$\delta>0$ and $\varepsilon\in (0,1)$. Then, the $n-(\beta+1)=\Omega(n)$ 
remaining bits cannot be sampled as all 1s
during any $e^{\Omega(\mu)}$ iterations 
with probability $1-e^{-\Omega(\mu)}$.
\end{lemma}

The proof makes use of the observation that 
the remaining bits are all 
sampled correctly
if the sum $\sum_{i\ge \beta+2}Y_{t,i}=n-(\beta+1)=\Omega(n)$, 
and by Chernoff-Hoeffding bound \citep{Dubhashi:2009:CMA}. 
We are ready to show our main result of the \umda on the \los function.

\begin{theorem}
\label{thm:umda-exponential-runtime}
The runtime of 
the UMDA  with the parent population size 
$\mu\ge c\log n$ for some sufficiently large constant 
$c>0$ and the offspring population size
$\lambda \le \mu e^{1-\varepsilon}/(1+\delta)$ for any constants
$\delta>0$ and $\varepsilon \in (0,1)$ 
is $e^{\Omega(\mu)}$ on the \los function with probability $1-e^{-\Omega(\mu)}$
and in expectation.
\end{theorem}

\begin{proof}
It suffices to show the high-probability statement as by the law of total expectation
\citep{bib:Motwani1995} the expected runtime 
is $e^{\Omega(\mu)}(1-e^{-\Omega(\mu)})=e^{\Omega(\mu)}$.
Consider phase 1 and phase 2 as mentioned above. We also assume
that phase 1 lasts for a polynomial number of iterations; otherwise, we are done and the theorem trivially holds. 

During phase 2,
we have observed that the random variable $Z_t$ always stays below 
$\beta$ for any $t\le e^{\Omega(\mu)}$ with high probability, while
the $\Omega(n)$ remaining bits cannot be sampled correctly
in any iteration $t\le e^{\Omega(\mu)}$ with probability $1-e^{-\Omega(\mu)}$ by
Lemma~\ref{lemma:los-unlikely-sample-remaining-blocks}.
Thus, the all-ones bitstring will be sampled with probability at most 
$e^{-\Omega(\mu)}$, the runtime
of the \umda on the \los function is 
$e^{\Omega(\mu)}$ with
probability $1-e^{-\Omega(\mu)}$, which completes the proof.
\end{proof}

\section{High selective pressure}
\label{sec:umda-los-high-sel-press}

When the selective pressure 
becomes higher such that
the value of $\alpha=\alpha(n)$ 
exceeds the problem instance size $n$, 
phase 1 would end when the $\mu$ fittest individuals 
are all-ones bitstrings, i.e., the global 
optimum has been found. 
In order for this to be the case, 
by (\ref{eq:alpha-value}) we obtain the inequality
$
\gamma^*/(1-\delta) \le \left(1-1/n\right)^n
$
for any constant $\delta\in (0,1)$. 
The right-hand side is at least $(1-\delta)/e$ for any
$n\ge (1+\delta)/\delta$ \citep{LehreOliveto2017}. If we choose 
the selective pressure $\gamma^*\le (1-\delta)^2/e$, 
the above inequality always holds. 
In this case, Dang and Lehre \citep{Dang:2015}
have already shown that the algorithm requires an 
$\bigO{n\lambda\log \lambda+n^2}$ expected runtime
on the function via the level-based theorem. 
We are now going to show a lower bound of 
$\Omega(\frac{n\lambda}{\log(\lambda-\mu)})$ 
on the expected runtime. 

\begin{lemma}
\label{lemma:distance-z-zstar}
For any $t\in \mathbb{N}$ that $\mathbb{E}[Z_t^*-Z_t]=\bigO{\log \mu}$.
\end{lemma}
\begin{proof}
Let $\delta_{t}:=Z_{t}^*-Z_{t}$. We pessimistically assume that 
the $Z_t$ 
first marginals are all set to one 
since we are only interested in a lower bound and this 
will speed up the optimisation process. We also define $\delta_t'$ to be the 
number of leading 1s of the fittest individual in a population of 
$\mu$ individuals each of length $n-Z_t-1$. By the law of total expectation, we get 
\begin{align*}
    \mathbb{E}[\delta_t\mid Z_t]
     &=(1+\mathbb{E}[\delta_t'\mid Z_t,X_{t,Z_t+1}=\mu])\cdot \Pr(X_{t,Z_t+1}=\mu\mid Z_t)\\
    &\le 1+\mathbb{E}[\delta_t'\mid Z_t,X_{t,Z_t+1}=\mu]
    =1+\mathbb{E}[\delta_t'\mid Z_t].
\end{align*}

We are left to calculate the expectation of $\delta_t'$, conditional on the 
random variable $Z_t$.
Let $f:=\los$. For simplicity, we also denote $(p_{i})_{i=1}^{n'}$ 
as the marginals of the bit positions
from $Z_t+2$ to $n$, respectively, where $n':=n-Z_t-1$. 
The probability of sampling an individual with $k$
leading 1s is $\Pr(f(x)=k)=(1-p_{k+1})\prod_{i=1}^{k}p_{i}$, then 
$\Pr(f(x)\le k)=\sum_{j=0}^k \Pr(f(x)=j)
=\sum_{j=0}^k(1-p_{j+1}) \prod_{i=1}^j p_{i}$. 
Furthermore, the probability that all $\mu$ 
individuals have at most $k$ leading 1s is 
$\Pr(\delta_t'\le k)=\prod_{q=1}^{\mu}\Pr(f(x^{(q)})\le k)$, and 
$\Pr(\delta_t'>k)=1-\Pr(\delta_t'\le k)$. 
Because $\mathbb{E}[Y]\le \sum_{i=0}^\infty \Pr(Y>i)$
for any bounded integer-valued random variable $Y$, we then get 
\begin{align*}
   \mathbb{E}[\delta_t'\mid (p_i)_{i},Z_t]
&\le \sum_{k=0}^{\infty}\left(1-\prod_{q=1}^{\mu}\sum_{j=0}^k (1-p_{j+1})
\prod_{i=1}^j p_i\right). 
\end{align*}
Note that by Lemma~\ref{lemma:los-expectation-of-Y}, 
each marginal $p_i$ has an expectation of 
$(1\pm o(1))(1/2)$. 
By the tower property of expectation, linearity of expectation and
independent sampling, we then obtain 
\begin{align*}
    \mathbb{E}[\delta_t'\mid Z_t]
    &=\mathbb{E}[\mathbb{E}[\delta_t'\mid (p_i)_i,Z_t]]\\
    &\le \sum_{k=0}^{\infty} (1-(1-2^{-(k+1)})^{\mu})+o(1)=\bigO{\log \mu}.
\end{align*}
The final bound follows from \citep{EISENBERG2008135},
which completes the proof.
\end{proof}

Lemma~\ref{lemma:distance-z-zstar} gives the important insight that the
random variables 
$Z_t$ and $Z_t^*$ only differ by a logarithmic additive term at any point in time in
expectation. Clearly, the global optimum is found when the random variable
$Z_t^*$ obtains the value of $n$. We can therefore alternatively
analyse the random variable $Z_t$ instead of $Z_t^*$. 
In other words, the random variable $Z_t$, starting from an initial value 
$Z_0$ given in Lemma~\ref{lemma:fittest-individual},
has to travel an expected distance of $n-\bigO{\log \mu}-Z_0$ bit
positions before
the global optimum is found. We shall make use 
of the additive drift theorem (for a lower bound) \citep{HE200359}
for a distance function $g(x)=n-x$ on the 
stochastic process $(Z_t)_{t\in \mathbb{N}}$. Let 
$\Delta_t:=g(Z_t)-g(Z_{t+1})=(n-Z_t)-(n-Z_{t+1})
=Z_{t+1}-Z_{t}$ be the single-step 
change (also called drift) in the value of the random variable $Z_t$. The following lemma provides an upper bound on the expected drift, 
which directly leads to a lower bound on the 
expected runtime.

\begin{lemma}
\label{lemma:expected-drift-UB}
For any $t\in \mathbb{N}$ and $Z_t\in [n-1]\cup \{0\}$, it holds that
$\mathbb{E}[\Delta_t\mid \mathcal{F}_t]=\bigO{\log (\lambda-\mu)}$.
\end{lemma}
\begin{proof}
Consider bit $i=Z_t+1$. 
The random variable $Z_t$ does not change in value
if $C_{t+1,i}<\mu$.
Thus, the maximum drift is obtained when 
$C_{t+1,i}\ge \mu$. In this case, we can express
$(\Delta_t\mid Z_t) \le 1+(\Delta_t'\mid Z_t)$, where the non-negative
$\Delta_t'\mid Z_t$ denotes the difference  between
the number of leading 1s of the $\mu$-th individual in 
iteration $t+1$ and the value $Z_t+1$. Here, we can take an alternative view that
$\Delta_t'\mid Z_t$ is stochastically
dominated by the number of leading 1s of the 
fittest individual in a population of $\lambda-\mu$ 
individuals each of length 
$n-Z_t-1$, sampled from a product distribution where each marginal 
has an expectation of $(1\pm o(1))/2$ (by Lemma~\ref{lemma:los-expectation-of-Y}).
Following \citep{EISENBERG2008135}, the fittest individual in this population
has $\bigO{\log (\lambda-\mu)}$ leading 1s in expectation.
\end{proof}

We are ready to show a lower bound on the expected 
runtime of the \umda on the \los function.

\begin{theorem}
\label{thm:lower-bound-umda-los}
The expected runtime of the \umda with a parent population size 
$\mu\ge c\log n$ for some sufficiently large constant $c>0$ and an
offspring population size $\lambda\ge \mu e/(1+\delta)^2$ 
where the problem instance size is 
$n\ge (1+\delta)/\delta$ for any constant  $\delta\in (0,1)$ 
is $\Omega(\frac{n\lambda}{\log(\lambda-\mu)})$ on the \los function.
\end{theorem}
\begin{proof}
Consider the drift $\Delta_t$ on the value of the random variable $Z_t$. 
By Lemma~\ref{lemma:expected-drift-UB} we get
$\mathbb{E}[\Delta_t\mid \mathcal{F}_t]=\bigO{\log (\lambda -\mu)}$.
Since the random variable 
$Z_t$ has to travel an expected distance of $n-\bigO{\log\lambda}-Z_0$ 
before the 
global optimum is found, the additive drift theorem 
shows that the expected number of iterations until
the global optimum is found is
$
\mathbb{E}[T\mid Z_0]=\bigO{(n-Z_0)/\log(\lambda-\mu)},
$
which by the towel rule and noting that
$\mathbb{E}[Z_0]=\bigO{\log \lambda}$ satisfies
$
\mathbb{E}[T]=\mathbb{E}[\mathbb{E}[T\mid Z_0]]
=\Omega(n/\log(\lambda-\mu))
$.
The proof is complete by noting that
there are $\lambda$ fitness evaluations in each iteration 
of the \umda.
\end{proof}

\section{LeadingOnes with prior noise}
\label{sec:umda-noise-los}

We consider a 
prior noise model and formally define the problem for any constant
$0<p<1$ as follows.
\begin{equation*}
F(x_1,\ldots,x_n)=\begin{cases}
    f(x_1,\ldots,x_n), &\text{w.p. }1-p, \text{ and}\\
    f(\ldots,1-x_i,\ldots), &\text{w.p. }p, \text{ where }i\sim \text{Unif}([n]).
    \end{cases}
\end{equation*}
We denote $F$ as the noisy fitness and $f$ as the actual 
fitness. For simplicity, we also 
denote $P_t$ as the population prior to noise.  
The same noise model is studied in 
\citep{Giessen2016,Droste2002,bib:Dang2015} 
for population-based \eas on the 
\om and \los functions. 

We shall make use of the level-based theorem
\citep[Theorem~1]{bib:Corus2016} and first
partition the search space $\mathcal{X}$ into $n+1$ 
disjoint subsets $A_0,\ldots,A_n$, where 
\begin{equation}
\label{def:levels-los}
A_j=\{x\in \mathcal{X}:\los(x)=j\}.
\end{equation}
We also denote $A_{\ge j}=\{x\in \mathcal{X} \mid \los(x)\ge j\}$.
We then need to verify three conditions (G1), (G2) and (G3) of the level-based theorem
\citep{bib:Corus2016}, where due to the presence of 
noise we choose the parameter 
$\gamma_0=\gamma^*/((1-\varepsilon)(1-p))$ for any 
constant $\varepsilon\in (0,1)$  to leverage the 
impact of noise in our analysis. The following 
lemma tells us the number of individuals
in the noisy population in iteration $t$ which has
fitness 
$F(x)=f(x)\ge j$.

\begin{lemma}
\label{lem:num-of-unaffected-inds}
Assume  that $|P_t\cap A_{\ge j}|\ge \gamma_0\lambda$, where 
$\gamma_0:=\gamma^*/((1-p)(1-\delta))$ for some constant $\delta\in (0,1)$. 
Then, there are at least $\mu$ individuals with the fitness 
$F(x)=f(x)\ge j$ in the noisy population 
with probability $1-e^{-\Omega(\mu)}$. Furthermore, 
there are
at most $\varepsilon\mu$ individuals with actual 
fitness $f(x)\le j-1$ and noisy fitness $F(x)\ge j$ for some small constant
$\varepsilon\in (0,1)$ with probability $1-e^{-\Omega(\mu)}$.
\end{lemma}
\begin{proof}
We take an alternative view on 
the sampling of the population and the application of noise. More specifically, 
we first sample the population, sort it in descending order according to
the true fitness, and then noise occurs at an individual 
with probability $p$. Because
noise does not occur at an individual w.p. $1-p$,
amongst the $\gamma_0\lambda$ individuals in 
levels $A_{\ge j}$,
in expectation there are 
$(1-p)\gamma_0\lambda
= \gamma^*\lambda/(1-\delta)=\mu/(1-\delta)$ 
individuals unaffected by noise. 
Furthermore, by a Chernoff bound
\citep{bib:Motwani1995},
there are at least $(1-\delta)\cdot \mu/(1-\delta)=\mu$  such 
individuals  for some constant 
$0<\delta<1$ with probability at least 
$1-e^{-(\delta^2/3)\cdot \mu/(1-\delta)}=1-e^{-\Omega(\mu)}$, which proves the 
first statement.

For the second statement, we only consider 
individuals with actual fitness $f(x)<j$ 
and noisy fitness $F(x)\ge j$ in the noisy population. If such an individual
    is selected when updating the model, it will introduce a 0-bit to the total number of 0s
    among the $\mu$ fittest individuals for the first $j$ bits. Let 
    $B$ denote the number of such individuals. There are at most $(1-\gamma_0)\lambda$ individuals with actual fitness $f(x)<j$, the probability that its noisy fitness is at least $F(x)\ge j$ is at most $p/n$ because a specific bit must be flipped in the prior noise model. Hence the
    expected number of these individuals is upper bounded by 
    \begin{equation}
    \label{eq:expectation-B}
        \mathbb{E}[B]\le (1-\gamma_0)\lambda p/n<\lambda p/n.
    \end{equation}
    We now show by a Chernoff bound
    that the event $B\ge \varepsilon\mu$ for a small constant 
    $\varepsilon\in (0,1)$ 
    occurs with probability at most $e^{-\Omega(\mu)}$. We shall rely on the fact that 
    $\lambda p/n\le \mu\varepsilon/2$ for sufficiently large $n$, which follows from the assumption
    $\mu/\lambda=\Theta(1)$. 
    We use the parameter $\delta:=\varepsilon\mu/\mathbb{E}[B]-1$,
    which by (\ref{eq:expectation-B}) and the assumption
    $\lambda p/n\le \varepsilon\mu/2$ satisfies 
    $\delta\ge \varepsilon\mu n/(p\lambda)-1\ge 1$.
    We also have the lower bound 
    $$
    \delta\cdot \mathbb{E}[B]=\varepsilon\mu-\mathbb{E}[B]
    \ge \varepsilon\mu-\lambda p/n \ge \varepsilon\mu/2.
    $$
    A Chernoff bound \citep{bib:Motwani1995} now gives the desired result 
    \begin{equation}
    \label{eq:prob-sampling-B-less-mu}
       \Pr(B\ge \varepsilon\mu)=\Pr(B\ge (1+\delta)\mathbb{E[B]})\le e^{-\delta \mathbb{E}[B]/3}=e^{-\varepsilon\mu/6},
    \end{equation}
    which completes the proof.
\end{proof}

We now derive upper bounds on the expected runtime of the \umda 
on \los in the noisy environment.

\begin{theorem}
Consider a prior noise model with parameter $p<1$.
The expected runtime of the \umda with a parent 
population size $\mu\ge c\log n$ for some sufficiently large constant 
$c>0$ where $n\ge 1/(3/4-\varepsilon)$ for some small constant 
$\varepsilon\in (0,3/4)$ 
and an offspring population size $\lambda\ge 4e(1+\delta)\mu$ 
is $\bigO{n\lambda\log \lambda + n^2}$ on the
\los function.
\end{theorem}
\begin{proof}
We will make use of the level-based analysis, 
in which we need to verify the three conditions in
the level-based theorem. 
Each level $A_j$ for $j\in [n]\cup\{0\}$ is
formally defined as in (\ref{def:levels-los}), 
and there are a total of $m:=n+1$ levels.

Condition (G1) assumes that 
$|P_t\cap A_{\ge j}|\ge \gamma_0\lambda$, and we are required to show that the 
probability of sampling an offspring in levels $A_{\ge j+1}$ 
in iteration $t+1$ is lower bounded by a value $z_j$. 
We choose the parameter
    $\gamma_0=\gamma^*/((1-\delta)(1-p))$ for any constant $\delta\in (0,1)$ and 
    the selective pressure $\gamma^*=\mu/\lambda$ (assumed to be constant).
    For convenience, we also partition the noisy population into four groups:
    \begin{enumerate}
        \item Individuals with the fitness $f(x)\ge j$ and $F(x)\ge j$.
        \item Individuals with the fitness $f(x)\ge j$ and $F(x)< j$.
        \item Individuals with the fitness $f(x)< j$ and  $F(x)\ge j$.
        \item Individuals with the fitness $f(x)< j$ and $F(x)< j$.
    \end{enumerate}
    
    By Lemma~\ref{lem:num-of-unaffected-inds}, there are at least $\mu$ individuals 
    in group 1 with probability $1-e^{-\Omega(\mu)}$. 
    The algorithm selects the $\mu$ fittest individuals according to the noisy fitness values to update the probabilistic model. Hence, unless the mentioned event does
    not happen, no individuals
    from group 2 or group 4 will be included when updating the model.
    
    We are now going to analyse how individuals from group 3 impact the marginal probabilities. Let $B$ denote the number of individuals in group 3. 
      We pessimistically assume that the algorithm uses all of the $B$ individuals in group 3 and $\mu-B$
    individuals chosen from group 1 when updating the model. 
    For all $i\in [j]$, let  $X_i$
    be the number of individuals in group 3 which has a 1-bits 
    in positions $1$ through $j$, except for one position $i$ where it has a 0-bit. 
    By definition, we then have 
    $
    \sum_{i=1}^j X_i =B.
    $
     The marginal probabilities after updating the model are
    \begin{equation}
    \label{eq:p-definition}
    p_{t,i}=\begin{cases}
    1-X_i/\mu, & \text{if }X_i>0,\\
    1-X_i/\mu-1/n, &\text{if }X_i=0.
    \end{cases}
    \end{equation}
    Following \citep{Boland-1983,Lehre2018PBIL,bib:Wu2017}, we lower bound the 
    probability of sampling an offspring $x$ 
    with actual fitness $f(x)\ge j$, by
    \begin{equation}
    \label{eq:lower-bound-prods}
    \prod_{i=1}^j p_{t,i}\ge \prod_{i=1}^j q_i,
    \end{equation}
    which holds for any vector $q:=(q_1,\ldots,q_j)$ which majorises the vector 
    $p:=(p_{t,1},\ldots,p_{t,j})$. Recall 
    (see \citep{gleser1975, Boland-1983}) that the vector 
    $q$ majorises the vector $p$ if for all $k\in [j-1]$
    \begin{displaymath}
    \sum_{i=1}^k q_i\ge \sum_{i=1}^k p_{t,i}, \text{ and }
    \sum_{i=1}^j q_i = \sum_{i=1}^j p_{t,i}.
    \end{displaymath}
    We construct such a vector $q$ which by the definition majorises the vector 
    $p$ as follows.
    \begin{displaymath}
    q_i=\begin{cases}
    1-1/n, & \text{if }i<j,\\
    \sum_{k=1}^j p_{t,k}-(1-1/n)(j-1), &\text{if }i=j.
    \end{cases}
    \end{displaymath}
  We now show that with high probability, the vector element 
  $q_j$ stays within the interval $[1-1/n-\varepsilon,1-1/n]$, i.e., $q_j$
  is indeed a probability. Since 
  $p_{t,i}\le 1-1/n$ for all 
  $i\le j$, we have the upper bound 
  $
  q_j\le (1-1/n)j-(1-1/n)(j-1)=1-1/n
  $.
  For the lower bound, 
  we note from (\ref{eq:p-definition}) 
  that $p_{t,i}\ge 1-X_i/\mu-1/n$ for all $i\le j$ and
  any $X_i\ge 0$, so we also obtain 
  \begin{align*}
       q_j &\ge \sum_{k=1}^j (1-X_i/\mu-1/n)-(1-1/n)(j-1) \\
       & =1-1/n-\sum_{k=1}^j X_i/\mu=1-1/n-B/\mu.
  \end{align*}
  By Lemma~\ref{lem:num-of-unaffected-inds},
  we have $B\le \varepsilon\mu$ for some small
  constant $\varepsilon\in (0,1)$ 
  with probability $1-e^{-\Omega(\mu)}$. Assume that this high-probability event
  actually happens, we therefore have 
  $q_j\ge 1-1/n-\varepsilon$. From this result, 
  the definition of the vector $q$ and (\ref{eq:lower-bound-prods}), we can conclude 
 that the probability of sampling in iteration
  $t+1$ an offspring $x$ with actual fitness $f(x)\ge j$ is
  \begin{displaymath}
  \prod_{i=1}^j p_{t,i}
  \ge \prod_{i=1}^j q_i
  \ge \left(1-\frac{1}{n}\right)^{j-1}\left(1-\frac{1}{n}-\varepsilon\right)
  \ge \frac{1}{4e}=\Omega(1)
  \end{displaymath}
  since $(1-1/n)^{j-1}\ge 1/e$ for any $n>0$, and by choosing
  $n\ge 1/(3/4-\varepsilon)$ for some positive constant $\varepsilon <3/4$.
  Because we also have $p_{t,j+1}\ge 1/n$, the probability of sampling an 
  offspring in levels $A_{\ge j+1}$ is at least $\Omega(1)\cdot (1/n)=\Omega(1/n)$.
  Thus, the condition (G1) holds with a value of $z_j=\Omega(1/n)$.
  
  For the condition (G2), we assume further that 
  $|P_t\cap A_{\ge j+1}|\ge \gamma\lambda$ for some value $\gamma\in (0,\gamma_0)$,
  and we are also required to show that the probability of sampling an offspring in
  levels $A_{\ge j+1}$ is at least $(1+\delta)\gamma$ for some small constant 
  $\delta\in (0,1)$. Because the marginal $p_{t,j+1}$ can be lower bounded 
  by $\gamma\lambda/\mu$,  
  the above probability can be written as follows.
  \begin{align*}
      \prod_{i=1}^{j+1}p_{t,i}
      \ge p_{t,j+1}\cdot \prod_{i=1}^j p_{t,i}
      &\ge \frac{\gamma\lambda}{\mu} 
      \cdot \frac{1}{4e}\ge (1+\delta)\gamma,
    \end{align*}
    where by choosing
    $\lambda/\mu\ge 4e(1+\delta)$ for some constant $\delta\in (0,1)$.
Thus, the condition (G2) of the level-based theorem is verified.
  
  The condition (G3) requires the offspring population size to satisfy
  \begin{align*}
  \lambda\ge \frac{4}{\gamma_0\delta^2}
    \ln\left(\frac{128m}{\delta^2\cdot \min_j\{z_j\}}\right)
  =\Omega\left(\frac{1-p}{\gamma^*}\log n\right).
  \end{align*}
  Having fully verified the three conditions (G1), (G2) and (G3), and noting
  that $\ln(\delta\lambda/(4+\delta z_j))<\ln(3\delta\lambda/2)$, the level-based 
  theorem   now guarantees an upper bound of $\bigO{n\lambda\log \lambda+n^2}$ on the
  expected runtime of the \umda on the noisy \los function.
  
  We note that our proof is not complete since
     throughout the proof we always assume the happening 
     of the following two events in each
     iteration of the \umda (see Lemma~\ref{lem:num-of-unaffected-inds}): 
     \begin{enumerate}
    \item[(A)] The number of individuals in group 1
    is at least $\mu$ with probability 
     $1-e^{-\Omega(\mu)}$.
     \item[(B)] The number of individuals in group 3 is $B\le \varepsilon\mu$ for some small
     constant $\varepsilon\in (0,3/4)$ 
     with probability $1-e^{-\Omega(\mu)}$.
     \end{enumerate}
     We call an iteration a \textit{success}
     if the two events happen simultaneously; otherwise, 
     we speak of a \textit{failure}.
     By the union bound, an iteration is a failure with
     probability at most $2e^{-\Omega(\mu)}$, 
     and a failure occurs at least once in a polynomial
     number of iterations 
     with probability at most $\text{poly}(n)\cdot 2e^{-\Omega(\mu)}$. If we
     choose the parent population size $\mu\ge c\log n$ for some sufficiently large
     constant $c>0$, then the above probability becomes 
     $\text{poly}(n)\cdot 2e^{-c\log n}\le n^{c'}$ for 
     some other constant $c'>0$. Actually, the upper bound 
     $\bigO{n\lambda\log \lambda+n^2}$ given by the 
    level-based theorem is conditioned
     on the event that there is no failure 
     in any iteration.
    We can obtain the (unconditionally) expected runtime by splitting the time into
    consecutive phases of length $t^*=\bigO{n\lambda\log \lambda+n^2}$, and
    by \citep[Lemma~5]{Lehre2018PBIL}
    the overall expected runtime is at most 
    $4(1+o(1))t^*=\bigO{n\lambda\log \lambda+n^2}$.
\end{proof}

As a final remark, we note that the exponential lower bound 
in Theorem~\ref{thm:umda-exponential-runtime} for the \los function
without noise should also hold
for the noisy \los function.

\section{Experiments}
\label{sec:experiments}

In this section, we provide an empirical study in order to see how
closely the theoretical results match the experimental results for reasonable
problem sizes, and to investigate a wider range of parameters.
Our analysis is focused on different regimes on the selective pressure 
in the noise-free setting.

\subsection{Low selective pressure}

We have shown in Theorem~\ref{thm:umda-exponential-runtime}
that when the selective pressure 
$\gamma_0\ge (1+\delta)/e^{1-\varepsilon}$ for any constants 
$\delta>0$ and $\varepsilon\in (0,1)$,
the \umda  requires
$2^{\Omega(\mu)}$ function evaluations to optimise the \los
function with high probability. 
We now choose 
$\delta=0.2$ and $\varepsilon=0.1$,  we then get 
$\gamma_0\ge (1+0.2)/e^{1-0.1}\approx 0.4879$. 
Thus, the choice
$\gamma_0=0.5$ should be sufficient to yield an exponential runtime.
For the population size, we experiment with three
different settings: $\mu=5\log n$ (small), 
$\mu=\sqrt{n}$ (medium) and $\mu=n$ (large) for a problem instance size
$n=100$. 
Substituting everything into (\ref{eq:alpha-value}) and 
(\ref{eq:beta-value}), we then get
$\alpha\approx47$ and $\beta\approx 87$. The numbers of leading 1s of the fittest
individual and the $\mu$-th individual in the sorted population (denoted
by random variables $Z_t^*$ and $Z_t$ respectively) are shown in Fig.~\ref{fig:umda-los-low} 
 over an epoch of 5000 iterations. The dotted blue lines denote
 the constant functions of $\alpha=47$ and 
 $\beta=87$. One can see that the random variable 
$Z_t$ keeps increasing until it reaches the value 
of $\alpha$ during the early stage and always stays well under value $\beta$ afterwards. 
Furthermore, $Z_t^*$ does not deviate too far from 
$Z_t$ that matches our analysis since the chance of sampling all ones
from the $n-\beta$ remaining bits is exponentially small.

\subsection{High selective pressure}

When the selective pressure is sufficiently high, 
that is, $\gamma_0\le (1-o(1))(1-\delta)/e$ for any constant 
$\delta\in (0,1)$, there is an upper bound 
$\bigO{n^2+n\lambda\log\lambda}$ on the expected runtime  \citep{Dang:2015}.  Theorem~\ref{thm:lower-bound-umda-los} yields a lower bound of 
$\Omega(\frac{n\lambda}{\log(\lambda-\mu)})$.
We start by looking at how the values of 
random variable $Z_t$ and $Z_t^*$ change over time.
Our analysis shows that it never decreases during
the whole optimisation course with overwhelming 
probability and eventually reaches the value of $n$. 
Similarly, we consider the three different settings for population size and also
note that our result holds for a parent population size 
$\mu\ge c\log n$, when the constant $c>0$ 
must be tuned carefully; in this experiment, we set $c=5$ (an integer larger than 3 
should be sufficient). We then get 
$\gamma^* \le (1-1/100)(1-0.1)/e \approx 0.1821$. Therefore,
the choice of $\gamma_0=0.1$ should be sufficient and we then get 
$\alpha\approx 160\gg n=100$. The experiment outcomes are shown
in Fig.~\ref{fig:umda-los-high-z}. 
The empirical behaviours of the two random 
variables match our theoretical analyses.

Furthermore, we are also interested in the average runtime of the algorithm. 
We run some experiments using the same settings for the population size 
where $n\in \{100,200,\ldots,1000\}$. 
For each value of $n$, the algorithms are run 
100 times, and the average runtime is computed. 
The empirical results are shown in 
Fig.~\ref{fig:umda-los-high-regression}. 
We then perform non-linear regression to
fit the power model $y\sim a\cdot n^b$ 
to the empirical data. The fittest model and its corresponding
coefficients $a$ and $b$ are also plotted. 
As seen in Fig.~\ref{fig:umda-los-high-regression}, 
the fittest models are all in the order of $n\lambda/\log(\lambda-\mu)$,
which matches the expected runtime
given in our theoretical analysis. 

\begin{figure}[t]
    \setlength{\abovecaptionskip}{0.00cm}
    \setlength{\belowcaptionskip}{0cm} 
    \centering
    \includegraphics[width=\linewidth, height=.9in]{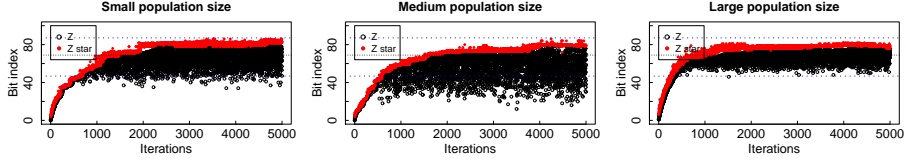}
    \caption{Low selective pressure 
    over long-range time.}
    \label{fig:umda-los-low}
\end{figure}
\begin{figure}[t]
    \setlength{\abovecaptionskip}{0.00cm}
    \setlength{\belowcaptionskip}{0cm} 
    \centering
    \includegraphics[width=\linewidth,height=.9in]{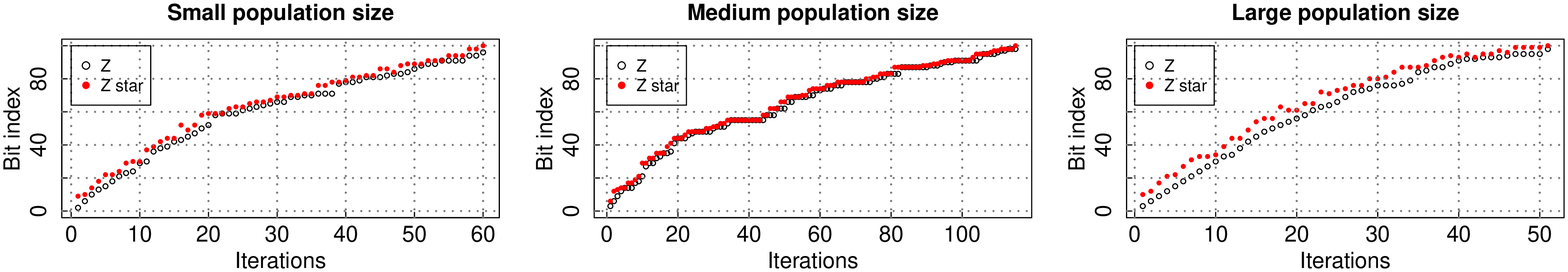}
    \caption{High selective pressure.}
    \label{fig:umda-los-high-z}
\end{figure}
\begin{figure}[t]
    \setlength{\abovecaptionskip}{0.00cm}
    \setlength{\belowcaptionskip}{0cm} 
    \centering
    \includegraphics[width=\linewidth,height=.9in]{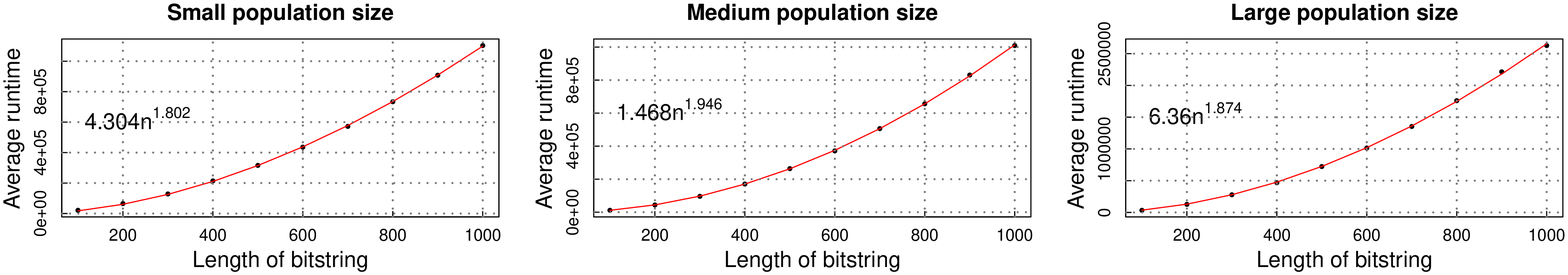}
    \caption{Average runtime under high selective pressure. 
   }
    \label{fig:umda-los-high-regression}
\end{figure}

\section{Conclusion and Future Work}
\label{sec:conclusion}

In this paper, we perform  rigorous analyses for the \umda (with margins)
on the \los function
in case of low selective pressure. We show that the algorithm
requires a $2^{\Omega(\mu)}$ runtime
with probability $1-2^{-\Omega(\mu)}$
and in expectation when $\mu\ge c\log n$ for a sufficiently 
large constant $c>0$ and 
$\mu/\lambda \ge (1+\delta)/e^{1-\varepsilon}$ for any constant 
$\delta>0$ and $\varepsilon\in (0,1)$.
The analyses reveal the limitations of the probabilistic model
based on probability vectors as the algorithm hardly stays at 
promising states for a long time. This leads the algorithm into a
non-optimal equilibrium state from which it is exponentially unlikely
to sample the optimal all-ones bitstring.
We also obtain the lower bound $\Omega(\frac{n\lambda}{\log(\lambda-\mu)})$
on the expected runtime when the selective pressure is sufficiently high. 
Furthermore, we study \umda in noisy optimisation setting for the
first time, where
noise is introduced to the \los function, causing 
a uniformly chosen bit is flipped with probability $p<1$. 
We show that an $\bigO{n^2}$ expected runtime still holds in this case for the
optimal offspring population size $\lambda=\bigO{n/\log n}$. 
Despite the simplicity of the noise model, this can be viewed as the first step towards broadening our understanding of the \umda in a noisy environment.

For future work, the \umda with an optimal 
offspring population size $\lambda=\bigO{n/\log n}$
needs $\bigO{n^2}$ expected time on the \los function \citep{Dang:2015}. In this case,
Theorem~\ref{thm:lower-bound-umda-los} yields a lower bound $\Omega(n^2/\log^2 n)$.
Thus, it remains open whether this gap of $\Theta(\log^2 n)$ 
could be closed in order to achieve a tight bound on the runtime. 
Another avenue for future work would be to investigate
the \umda under a posterior noise model.


\section*{References}
\bibliography{references}

\end{document}